\documentclass[10pt,journal,compsoc]{IEEEtran}
\PassOptionsToPackage{numbers, compress}{natbib}

\newcommand{\tr}[1]{\textrm{#1}}
\newcommand{\mr}[1]{\mathrm{#1}}

\newcommand{\mc}[1]{\mathcal{#1}}
\newcommand{\mf}[1]{\mathsf{#1}}

\newcommand{\bp}{\boldsymbol{p}}
\newcommand{\bP}{\boldsymbol{P}}
\newcommand{\bq}{\boldsymbol{q}}
\newcommand{\bQ}{\boldsymbol{Q}}

\newcommand{\bx}{\boldsymbol{x}}

\newcommand{\bz}{\boldsymbol{z}}

\newcommand{\btheta}{\boldsymbol{\theta}}

\newcommand{\figref}[1]{Fig.~\ref{#1}}

\newcommand{\ie}{i.e.,~} 		
\newcommand{\eg}{e.g.,~}	

\newcommand{\argmax}{\mathop{\mr{argmax}}}
\newcommand{\argmin}{\mathop{\mr{argmin}}}

\newcommand{\set}[1]{\{#1\}}

\newcommand{\cd}{\cdot}
\newcommand{\ld}{\ldots}

\newcommand{\e}{\mr{e}}

\newcommand{\mcH}{\mc{H}}
\newcommand{\mcI}{\mc{I}}

\newcommand{\mcO}{\mc{O}}

\newcommand{\mcW}{\mc{W}}
\newcommand{\mcX}{\mc{X}}

\newcommand{\mcZ}{\mc{Z}}

\newcommand{\mfK}{\mf{K}}
\newcommand{\mfX}{\mf{X}}

\newcommand{\Real}{\mathbb{R}}		
\newcommand{\W}{\mf{W}}              	
\ifCLASSOPTIONcompsoc
  \usepackage[nocompress]{cite}
\else
  \usepackage{cite}
\fi

\usepackage[utf8]{inputenc} 
\usepackage[T1]{fontenc}    
\usepackage{hyperref}       
\usepackage{url}            
\usepackage{booktabs}       
\usepackage{amsfonts}       
\usepackage{nicefrac}       
\usepackage{microtype}      
\usepackage{float}
\usepackage{epsfig}
\usepackage{graphicx,psfrag}
\usepackage{amsmath}
\usepackage{amssymb}
\usepackage{array}
\usepackage{multirow}
\usepackage{subfig} 
\usepackage[scientific-notation=true]{siunitx}
\usepackage[ruled,vlined,linesnumbered]{algorithm2e}
\usepackage{cite}
\usepackage{amsthm}
\usepackage[usenames, dvipsnames]{color}
\usepackage{caption}

\title{Deep clustering: On the link between discriminative models and K-means}

\author{Mohammed Jabi, Marco Pedersoli, Amar Mitiche and Ismail Ben Ayed
\IEEEcompsocitemizethanks{\IEEEcompsocthanksitem M. Jabi is with Ultra Electronics-TCS, Montreal, Canada.

\IEEEcompsocthanksitem M. Pedersoli and I. Ben Ayed are with ETS Montreal, Canada.
\IEEEcompsocthanksitem A. Mitiche is with INRS, Montreal, Canada. 
}
}

\newtheorem{proposition}{Proposition}

\begin{document}

\IEEEtitleabstractindextext{%
\begin{abstract}

In the context of recent deep clustering studies, discriminative models dominate the literature and report the most competitive performances. These models learn a deep discriminative neural network classifier in which the labels are latent. Typically, they use multinomial logistic regression posteriors and parameter regularization, as is very common in supervised learning. It is generally acknowledged that discriminative objective functions (e.g., those based on the mutual information or the KL divergence) are more flexible than generative approaches (e.g., K-means) in the sense that they make fewer assumptions about the data distributions and, typically, yield much better unsupervised deep learning results. On the surface, several recent discriminative models may seem unrelated to K-means. This study shows that these models are, in fact, equivalent to K-means under mild conditions and common posterior models and parameter regularization. We prove that, for the commonly used logistic regression posteriors, maximizing the $L_2$ regularized mutual information via an approximate alternating direction method (ADM) is equivalent to minimizing a soft and regularized K-means loss. Our theoretical analysis not only connects directly several recent state-of-the-art discriminative models to K-means, but also leads to a new soft and regularized deep K-means algorithm, which yields competitive performance on several image clustering benchmarks.  

\end{abstract}

}

\maketitle

\IEEEdisplaynontitleabstractindextext

\IEEEpeerreviewmaketitle

\section{Introduction}

\IEEEPARstart{O}{ne} of the most fundamental unsupervised learning problem, clustering aims at grouping data into categories. Obtaining meaningful categorical representations of data without supervision
is fundamental in a breadth of applications of data analysis and visualization. With the excessive amounts of high-dimensional data (e.g., images) routinely collected everyday, the problem 
is currently attracting substantial research interest, in both the learning and computer vision communities. 

Clustering performance heavily depends on the structure of the input data. Therefore, representation learning methods, which encode the original data in feature spaces where the grouping tasks become much easier, are 
widely used in conjunction with clustering algorithms. Typically, feature learning and clustering are performed sequentially \cite{Trigeorgis2014}. However, with the success of deep neural networks (DNNs), a large number of recent studies, \eg \cite{Caron2018,Yang17,Dizaji17,Hu17,Jiang16,Springenberg15,Yang16,Xie16}, investigated joint learning of feature embedding (via DNNs) and estimation of latent cluster assignments (or labels). Commonly, these recent models are stated as the optimization of objective functions that integrate two types of losses: (1) a clustering loss, which depends on both latent cluster assignments and deep network parameters and, (2) a reconstruction loss as a data-dependent regularization, e.g., via an auto-encoder \cite{Dizaji17}, to prevent the embedding from over-fitting.  

 Clustering objectives fall into two main categories, generative, \eg $\mbox{K}$-means and Gaussian Mixture Models \cite{Biernacki00}, and  discriminative, \eg graph clustering clustering \cite{ShiMalik2000,kernelcut,DensityBias} and information-theoretic models \cite{Krause10,Bridle92}. Generative objectives explicitly model the density of data points within the clusters via likelihood functions, whereas discriminative objectives learn the decision boundaries in-between clusters via conditional probabilities (posteriors) over the labels given the inputs. In the context of recent deep clustering models, discriminative objectives dominate the literature and report the most competitive performances \cite{Dizaji17,Hu17,Xie16,Springenberg15}. For instance, \cite{Hu17,Springenberg15} learned deep discriminative neural network classifiers that maximize the mutual information (MI) between data inputs and latent labels (cluster assignments), following much earlier MI-based clustering works \cite{Bridle92,Krause10}. 
 In another very recent line of deep discriminative clustering investigations, e.g., \cite{Xie16,Dizaji17}, the problem is addressed  by introducing auxiliary target distributions, which can be viewed as latent probabilistic point-to-cluster assignments. Then, it is stated as the minimization of a mixed-variable objective containing the Kullback-Leibler (KL) divergence between these auxiliary targets and the posteriors of a discriminative deep network classifier, typically expressed as standard multilogit regression functions \cite{Dizaji17,Krause10}. The minimization is carried out by alternating two sub-steps, until convergence. The first sub-step fixes the network parameters and optimizes the objective w.r.t the targets. The second fixes target assignments, and optimizes the objective w.r.t  network parameters. Conveniently, this sub-step takes the form of standard supervised classifiers, in which the ground-truth labels are given by the latent auxiliary targets. The KL divergence is used in conjunction with other terms, to favor balanced partitions and to regularize model parameters. 

Generative models were also investigated in the context of deep clustering \cite{Yang17,Caron2018,Jiang16}. For instance, in \cite{Yang17,Caron2018}, a DNN is trained with a loss function that includes the standard $\mbox{K}$-means clustering objective. However, in the literature, it is commonly acknowledged that discriminative models are more flexible in the sense that they make fewer assumptions 
about data distributions and, typically, yield better unsupervised learning results, \eg
\begin{quote}
``{\em Generally it has been argued that the discriminative models often have better results compared to their generative counterparts}" \cite{Dizaji17}
\end{quote}
\begin{quote}
`` \ld {\em discriminative clustering techniques represent the boundaries or distinctions between categories. Fewer assumptions about the nature of categories are made, making these methods powerful and flexible in real world applications}'' \cite{Krause10}
\end{quote}
Furthermore, the results reported in the literature suggest that the performances of discriminative models are significantly better. For instance, the DEPICT model in \cite{Dizaji17}, which is based on the KL-divergence between multilogit regression posteriors and targets, reports a state-of-the-art performance on MNIST nearly approaching supervised learning performance. This discriminative model outperforms significantly the K-means loss investigated recently in the deep clustering model in \cite{Yang17}, with $14\%$ difference in accuracy; see Table \ref{tab:results}. On the surface, the several recent discriminative models based on either the MI or KL objectives, e.g., \cite{Dizaji17,Hu17,Springenberg15}, may seem completely unrelated to $\mbox{K}$-means. Our study shows that they are, in fact, equivalent to ${\mbox K}$-means under mild conditions and commonly used posterior models and parameter regularization. The following lists the main contributions of this work.   

\begin{itemize}
\item For the commonly used logistic regression posteriors, we prove that maximizing the $L_2$ regularized mutual information via an approximate alternating direction method (ADM) 
is equivalent to a soft and regularized K-means loss (Proposition \ref{prop:link-to-K-means}). 

\item We establish the link between state-of-the-art KL-based models, e.g., DEPICT \cite{Dizaji17}, and the standard mutual information objective \cite{Krause10}, which is used in a number recent deep clustering works \cite{Hu17,Springenberg15}. In particular, we show that optimizing the KL objective, in conjunction with a balancing term, can be viewed as an approximate ADM solution for optimizing the mutual information.

\item We give theoretical results that connect directly several recent discriminative formulations to K-means. Furthermore, this leads to a new soft and regularized version of deep K-means, which has approximately the same competitive performances as state-of-the-art discriminative algorithms on several benchmarks (Table \ref{tab:results}). 
\end{itemize}   

\section{Deep discriminative clustering models}

 Let $\mcX=\big\{\bx_1,\ld, \bx_N \big\}$ be an unlabeled data set composed of $N$ samples, each of dimension $d_x$, \ie $\bx_i\in \Real^{d_x}$. The purpose is to cluster the $N$ samples into $K$ categories (clusters). The data samples are embedded into a feature space $\mcZ=\big\{\bz_1,\ld, \bz_N \big\}$ using a mapping $\phi_\mcW: \mcX \rightarrow \mcZ$, where $\mcW$ are learnable  parameters and $\bz_i\in \Real^{d_z}$, with $d_z<<d_x$, \ie the dimensionality of $\mcZ$ is much smaller than $\mcX$. In recent deep clustering models, as in \cite{Dizaji17,Xie16,Springenberg15}, for instance, the embedding function $\phi_\mcW$ is learned jointly with latent cluster assignments (or labels) using a Deep Neural Network (DNN), in which case $\mcW$ denotes the set of network parameters. These models are stated as the optimization of an objective that integrates two types of loss terms: (1) a clustering loss, which depends on both latent cluster assignments and network parameters, and, (2) a reconstruction loss ${\cal R}(\mcZ)$ as a data-dependent regularization, e.g., via an auto-encoder \cite{Dizaji17}, to prevent the embedding from over-fitting. While all the recent deep clustering objectives discussed in the following used a reconstruction loss, we will focus only on the clustering losses in this section for the sake of clarity; we will discuss a reconstruction loss in more detail in section \ref{sec.experiments}.

\subsection{Mutual information}

Following the works in \cite{Bridle92,Krause10}, maximizing the mutual information between data inputs and latent cluster assignments is commonly used in discriminative clustering. Also, very recently, the concept is revisited in several deep clustering studies, e.g., \cite{Hu17,Springenberg15}, which learned discriminative neural network classifiers that maximize the mutual information and obtained competitive 
performances. In general, the problem amounts to maximizing the following clustering loss:  
\begin{align}\label{eq:MI}
\mcI(\mfX,\mfK)= \mcH(\mfK)-\mcH(\mfK|\mfX),
\end{align}
where $\mcH(\cd)$ and $\mcH(\cd|\cd)$ are the entropy and conditional entropy, respectively. $\mfK \in \set{1,\ld,K}$ and $\mfX \in \mcX$ denote random variables for cluster assignments (latent labels) and data samples, respectively. The objective is to learn a conditional probability (posterior) over the labels given the input data, which we denote $p_{ik}$. The marginal distribution of labels can be estimated as follows \cite{Krause10}:
\begin{align}
\hat{p}_k=\frac{1}{N}\sum_{i=1}^{N}p_{ik}
\end{align}
Thus, the entropy terms appearing in the mutual information can be expressed with the posteriors as follows \cite{Krause10}:
\begin{align}
\mcH(\mfK)&=-\sum_{k=1}^{K}\hat{p}_k \log\Big(\hat{p}_k \Big)\\
\mcH(\mfK|\mfX)&=-\frac{1}{N}\sum_{i=1}^{N}\sum_{k=1}^{K}p_{ik} \log\big(p_{ik}\big)
\end{align}
Minimizing the conditional entropy of the posteriors, $\mcH(\mfK|\mfX)$, inhibits the uncertainty associated to the assignment of labels to each data point. Each point-wise conditional entropy in the sum reaches its minimum
when a single label $k$ has the maximum posterior for point $i$, i.e., $p_{ik} = 1$, whereas each of the other labels verifies $p_{ij} = 0, j \neq k$. 
In the semi-supervised setting, it is well known that this conditional entropy models effectively the {\em cluster assumption} \cite{Grandvalet2004}: The decision boundaries of the discriminative model should not occur at dense regions of the inputs. However, using this term alone in the unsupervised setting yields degenerate solutions, in which decision boundaries are removed \cite{Krause10}. Maximizing the entropy of the marginal distribution of labels, $\mcH(\mfK)$, avoids degenerate solutions as it biases the results towards balanced partitions\footnote{
Notice that $\mcH(\mfK)$ is equal up to an additive constant to the Kullback-Leiber (KL) divergence between the label distribution and the uniform distribution: $\mbox{KL}((\hat{p}_k, k = 1, \dots, K) \Arrowvert (\hat{u}_k, k = 1, \dots, K))$, with $\hat{u}_k = 1/K ~ \forall k$. Also, note that it is possible to encourage label distribution $\hat{p}_k$ to match any prior distribution $\hat{d}_k$, not necessarily uniform, simply by using  $\mbox{KL}((\hat{p}_k, k = 1, \dots, K)\Arrowvert (\hat{d}_k, k = 1, \dots, K)$ \cite{Krause10,Hu17}.}.

Finally, one has to choose a parametric model for posteriors $p_{ik}$, e.g., the widely used multilogit regression function \cite{Krause10,Dizaji17}:
\begin{align}\label{eq:p_expression}
p_{ik}\varpropto\exp(\btheta_k^{T}\bz_i+b_k),
\end{align}
where $\mcO=\set{\btheta_1,\ld,\btheta_K,b_1,\ld, b_K}$ is the set of weight vectors $\btheta_k$ and bias values $b_k$ for each cluster $k$. We note here that $p_{ik}$ is also related to the DNN parameters $\mcW$, \ie $p_{ik}\equiv p_{ik}(\btheta_k,b_k,\mcW)$,  since $\bz_i=\phi_\mcW(\bx_i)$. In the reminder of the paper, we will use probability-simplex vectors $\bp_i \in [0, 1]^K$ to denote $(p_{i1}, \dots, p_{iK})^t$ 
and matrix $\bP = (\bp_{i1}, \dots, \bp_{iK}) \in [0, 1]^{K \times N}$ to denote the posteriors of all data points. To simplify the notation, we will omit the explicit dependence of posteriors $p_{ik}$ on model parameters $\{\mcO, \mcW\}$.

\subsection{The KL divergence and auxiliary targets}

In another very recent line of deep discriminative clustering investigations, e.g., \cite{Xie16,Dizaji17}, the problem is stated by introducing auxiliary target distributions
$\bq_i = (q_{i1}, \dots, q_{iK})^t \in [0, 1]^K$, which are latent probabilistic point-to-cluster assignments within the simplex. Then, the problem is formulated as the minimization 
the KL divergences between these auxiliary targets and the posteriors of a discrimintaive deep network classifier, which we denote $\bp_i$, as earlier in the case of the mutual information.
Conveniently, in this case, the sub-problem of optimizing w.r.t the network parameters takes the form of a standard supervised classifier, in which the ground truth labels 
are given by the auxiliary targets. For instance, the recent state-of-the-art model in \cite{Dizaji17}, referred to as DEPICT, follows from minimizing the KL divergence and a term that encourages balanced cluster assignments, subject to simplex constraints: 
\begin{align}\label{eq:kl.disc}
&\min_{\Phi,\bQ}~\mbox{KL}(\bQ\Arrowvert \bP) + \gamma \sum_{k=1}^{K}\hat{q}_{k}\log\Big(\hat{q}_{k} \Big) \nonumber \\ 
&~~\text{s.t.}~~ \bq_i^t {\mathbf 1} = 1; \bq_i \geq 0 ~ \forall i
\end{align}
where matrix $\bQ = (\bq_{i1}, \dots, \bq_{iK}) \in [0, 1]^{K \times N}$ contains the targets for all points, $\Phi=\{\mcO,\mcW\}$ and $\hat{q}_k=\frac{1}{N}\sum_{i=1}^{N}q_{ik}$ is the empirical distribution 
of the target assignments. The KL divergence is defined as:
\begin{align}\label{eq:KL}
 \mbox{KL}(\bQ\Arrowvert \bP)=\frac{1}{N} \sum_{i=1}^{N}\sum_{k=1}^{K}q_{ik} \log\Big(\frac{q_{ik}}{p_{ik}}\Big). 
\end{align}
The model in \eqref{eq:kl.disc} depends on two different types of variables: auxiliary targets $\bQ$ and classifier parameters $\Phi$. Therefore, it is solved by alternating two sub-steps, until convergence:
\begin{itemize}
\item {\bf Parameter-learning step}: This step fixes target assignments $\bQ$ and optimizes \eqref{eq:kl.disc} w.r.t network parameters $\Phi$. Notice that, ignoring constant terms, this sub-step 
becomes equivalent to a cross-entropy loss, exactly as in standard supervised classifiers, with ground-truth labels given by fixed targets $\bQ$:  
\begin{align}\label{eq:param-learn-depict}
\min_{\Phi}~-\frac{1}{N}\sum_{i=1}^{N}\sum_{k=1}^{K}q_{ik} \log p_{ik}
\end{align}
\item {\bf Target-estimation step}: This sub-step finds the target variable $\bQ$ that minimizes \eqref{eq:kl.disc}, with the network parameters fixed. 
Setting the approximated gradient equal to zero, it is easy to show that the optimal solution is given by \cite{Dizaji17}: 
\begin{align}\label{eq:q.DEPICT}
q_{ik} \varpropto \frac{p_{ik}}{\Big(\sum_{i'=1}^{N}p_{i'k}\Big)^{1/2}}
\end{align} 
\end{itemize}

{The above algorithm alternates two steps until convergence: (1) updates of network parameters $\Phi$ via back-propagation and stochastic gradient descent (SGD) corresponding to the standard cross-entropy loss; and (2) updates of target assignments $q_{ik}$, according to closed-form solution \eqref{eq:q.DEPICT}. While a direct maximization of the MI is based on SGD solely, this alternating scheme has an additional computational load of $O(NK)$, which comes from target-variable updates in \eqref{eq:q.DEPICT}. The computational load associated with updates \eqref{eq:q.DEPICT} is marginal in comparison to the load associated with SGD network training. In practice, the training time associated with this alternating ADM scheme is approximately of the same order as a direct SGD maximization of the MI.  
}

The following proposition establishes the link between the state-of-the-art DEPICT model in \eqref{eq:kl.disc}, which was introduced recently in \cite{Dizaji17}, and the standard mutual information objective in \eqref{eq:MI}, which was used in a number of other recent deep clustering works \cite{Hu17,Springenberg15}.
\begin{proposition}
\label{Prop:ADM-KL}
Alternating steps \eqref{eq:param-learn-depict} and \eqref{eq:q.DEPICT} for optimizing mixed-variable objective \eqref{eq:kl.disc} can be viewed as an approximate Alternating Direction Method (ADM)\footnote{The most basic form of the ADM approach transforms a single-variable problem of the form $\min_{x} u(x) + v(x)$ into a constrained two-variable problem of the form $\max_{x,y} u(x) + v(y) ~\text{s.t.}~ x = y$. This splits the original problem into two easier sub-problems, alternating optimization over variables $x$ and $y$.} \cite{boyd2011distributed} for maximizing the mutual information $\mcI(\mfX,\mfK)$ in \eqref{eq:MI} via the following constrained decomposition of the problem:
\begin{align}\label{eq:objective.MI-ADM}
& \max_{\Phi,\bQ}~\frac{1}{N} \sum_{i=1}^{N}\sum_{k=1}^{K}{q}_{ik} \log(p_{ik}) - \sum_{k=1}^{K}\hat{q}_{k}\log\Big(\hat{q}_{k} \Big)  \nonumber \\
&~~\mbox{s.t.}~~ \bQ=\bP; ~ \bq_i^t {\mathbf 1} = 1;~ \bq_i \geq 0 ~ \forall i
\end{align}
\end{proposition}

\begin{proof}
It is easy to see that equality-constrained problem \eqref{eq:objective.MI-ADM} is an ADM decomposition of the mutual information maximization in \eqref{eq:MI}. Notice that, when constraint $\bQ=\bP$ is satisfied, one can replace each auxiliary target ${q}_{ik}$ in the objective of \eqref{eq:objective.MI-ADM} by posterior ${p}_{ik}$, which yields exactly the mutual information in \eqref{eq:MI}:
\begin{align}
\label{eq:MI-2}
\mcI(\mfX,\mfK) = \frac{1}{N} \sum_{i=1}^{N}\sum_{k=1}^{K}{p}_{ik} \log(p_{ik}) - \sum_{k=1}^{K}\hat{p}_{k}\log\Big(\hat{p}_{k} \Big)
\end{align}
Rather than optimizing directly mutual information \eqref{eq:MI-2} with respect to the parameters of posteriors $\bP$, ADM splits the problem into two sub-problems by introducing auxiliary 
variable $\bQ$ and enforcing $\bQ=\bP$. Now, notice that one can solve constrained problem \eqref{eq:objective.MI-ADM} with a {\em penalty} approach. This replaces constraint $\bQ=\bP$ 
by adding a term to the objective, which penalizes some divergence between $\bQ$ and $\bP$, e.g., KL\footnote{KL is non-negative and is equal to zero if and only if the two distributions are equal.}:
\begin{align}\label{eq:objective.MI-ADM-KL}
& \max_{\Phi,\bQ}~ \frac{1}{N} \sum_{i=1}^{N}\sum_{k=1}^{K}{q}_{ik} \log(p_{ik}) - \sum_{k=1}^{K}\hat{q}_{k}\log\Big(\hat{q}_{k} \Big) -  \mbox{KL}(\bQ\Arrowvert \bP) \nonumber \\
&~~\text{s.t.}~~ \bq_i^t {\mathbf 1} = 1;~ \bq_i \geq 0 ~ \forall i
\end{align}

This is closely related to the principle of ADMM (Alternating Direction Method of Multipliers) \cite{boyd2011distributed}, except that KL is not a typical choice for a penalty to replace the equality constraints. Typically, ADMM methods use multiplier-based quadratic penalties for enforcing the equality constraint (also referred to as augmented Lagrangian). {We will discuss more details on maximizing the MI directly or via an alternating direction method in Sections \ref{ADM-properties} and \ref{sec.experiments}. Furthermore, we will discuss the standard approach based on multiplier-based quadratic penalties and its link to the KL penalties for simplex variables.} 

Using expression \eqref{eq:KL} of KL in the objective of \eqref{eq:objective.MI-ADM-KL}, and after some manipulations, we can show that the problem in \eqref{eq:objective.MI-ADM-KL} is equivalent to:
\begin{align}\label{eq:objective.MI-ADM-KL-2}
& \min_{\Phi,\bQ}~ \mbox{KL}(\bQ\Arrowvert \bP) + \frac{1}{2} \sum_{k=1}^{K}\hat{q}_{k}\log\Big(\hat{q}_{k} \Big) + \frac{1}{2} \mbox{H}(\bQ)   \nonumber \\
&~~\text{s.t.}~~ \bq_i^t {\mathbf 1} = 1;~ \bq_i \geq 0 ~ \forall i
\end{align}
where $\mbox{H}(\bQ) = - \frac{1}{N} \sum_{i=1}^{N}\sum_{k=1}^{K}{q}_{ik} \log(q_{ik})$ is the entropy of auxiliary variable $\bQ$. Notice that the first two terms in 
\eqref{eq:objective.MI-ADM-KL-2} correspond to the DEPICT model in \eqref{eq:kl.disc} for $\gamma=1/2$. The last term ${\mbox H}(\bQ)$ encourages peaked auxiliary distributions
$\bq_i=(q_{i1}, \dots, q_{iK})$. Each point-wise entropy in the sum reaches its minimum at one vertex of the simplex: a single label $k$ has the maximum target variable for point $i$, i.e., $q_{ik} = 1$, whereas
each of the other variables verifies $q_{ij} = 0, j \neq k$. Term $\mbox{H}(\bQ)$ is close to zero near the vertices of the simplex (peaked distributions $\bq_i$). Therefore, the mutual information objective we 
obtained in model \eqref{eq:objective.MI-ADM-KL-2}, which we refer to as MI-ADM, can be viewed as an approximation of the DEPICT model in \eqref{eq:kl.disc}. In fact, as we will see in our experiments, the additional 
entropy term in \eqref{eq:objective.MI-ADM-KL-2}, ${\mbox H}(\bQ)$, has almost no effect on the results: DEPICT and MI-ADM have approximately the same performances; see Table \ref{tab:results}.  
\end{proof}

With the model parameters fixed, setting the approximated gradient of \eqref{eq:objective.MI-ADM-KL-2} w.r.t the target variables equal to zero, we obtain the following updates:
\begin{align}\label{eq:optimal.q.Disc}
q_{ik} \varpropto \frac{p_{ik}^2}{\Big(\sum_{i'=1}^{N}p_{i'k}^2\Big)^{1/2}}
\end{align}  
Notice that these updates are slightly different from the DEPICT updates in \eqref{eq:q.DEPICT}, due to additional entropy term ${\mbox H}(\bQ)$. It is worth noting that the recent deep discrimintaive
clustering algorithm in \cite{Xie16} updated $q_{ik}$ as follows: 
\begin{align}\label{eq:optimal.D.P}
q_{ik} \varpropto \frac{p_{ik}^2}{\Big(\sum_{i'=1}^{N}p_{i'k}\Big)}
\end{align}
This expression was found \emph{experimentally}, and was not based on a formal statement of the problem.

\section{Deep K-means}\label{sec.generative}

The standard generative K-means objective, integrated with a reconstruction loss, was recently investigated in the context of deep clustering \cite{Yang17, Aljalbout18}. In this case, a DNN is trained 
with a loss function that includes the classical $\mbox{K}$-means clustering objective, which takes the following form:
\begin{align}\label{eq:objective.k.means}
\sum_{i=1}^{N}\sum_{k=1}^{K}s_{ik} \Arrowvert \bz_{i}-\bold{\mu}_k \Arrowvert^2 ~~\text{s.t.}~~ \sum_{k=1}^{K}s_{ik}=1; ~~ s_{i,k} \in \{0, 1\} ~ \forall i, k
\end{align}
where $\mu_k$ is the cluster prototype (mean of features $\bz_{i}$) and $s_{ik}$ is a binary integer variable for assigning data point $i$ to cluster $k$: $s_{ik} = 1$ when point $i$ is assigned to 
cluster $k$, and $s_{ik} = 0$ otherwise. Similarly to earlier, $\bz_{i}$ denotes features that are learned jointly with clustering via an additional reconstruction loss ${\cal R}(\mcZ)$. On the surface, the discriminative mutual information objective in Eq. \eqref{eq:MI-2} and its ADM approximation in the DEPICT model in Eq. \eqref{eq:kl.disc} may seem completely different from the $K$-means 
loss in \eqref{eq:objective.k.means}. The following proposition shows that they are, in fact, equivalent under mild conditions. 

\begin{proposition}
\label{prop:link-to-K-means}
For balanced partitions and multiclass logistic regression posteriors of the form in \eqref{eq:p_expression}, ADM maximization of a regularized mutual information defined by
\begin{align}\label{eq:RMI}
\mcI(\mfX,\mfK)-\lambda \sum_{k=1}^{K} \btheta_k^{T}\btheta_k,
\end{align}
is equivalent to the minimization of the following \emph{regularized soft} $K$-means loss function\footnote{We omitted simplex constraints $\bq_i^t {\mathbf 1} = 1$ and  $\bq_i \geq 0 ~ \forall i$ in both Eqs. \eqref{eq:RMI} and \eqref{eq:R-soft-kmeans}. This simplifies the presentation without causing any ambiguity.}:
\begin{align}
\label{eq:R-soft-kmeans}
\sum_{i=1}^{N}\sum_{k=1}^{K} q_{ik} \Arrowvert  \bz_i- \btheta_k'\Arrowvert^2 + \lambda K \sum_{i=1}^{N} \sum_{k=1}^{K}q_{ik} \log(q_{ik}) - \sum_{i=1}^{N} \bz_i^{T} \bz_i, 
\end{align}
where $\lambda \in \Real$ is the regularization parameter and $\btheta_k'$ is a soft cluster prototype (mean) defined by: 
\begin{align}
\btheta'_k=\frac{\sum_{i=1}^{N}q_{ik}\bz_i}{ \sum_{i=1}^{N}q_{ik}}.
\end{align}
\end{proposition}

The function including the first two terms in \eqref{eq:R-soft-kmeans} can be viewed as a {\em soft} K-means objective. The first term corresponds 
exactly to \eqref{eq:objective.k.means}, except that the integer constraints on assignment variables are relaxed: $s_{ik} \in \{0,1\}$ are hard 
assignments (vertices of the simplex) whereas $q_{ik} \in [0,1]$ are soft assignments (within the simplex). The second term in \eqref{eq:R-soft-kmeans} is 
a negative entropy, which favors assignment softness. It reaches its maximum and vanishes (i.e., becomes equal to zero) for hard binary assignments $q_{i,k} \in \{0, 1\}$: 
at the vertices of the simplex, the function including the first two terms in \eqref{eq:R-soft-kmeans} becomes exactly the hard K-means objective in \eqref{eq:objective.k.means}. 
It is also worth noting that optimizing this soft K-means objective, with features $\bz_i$ fixed, yields {\em softmin} K-means updates that are known in the literature; see \cite[p. 289]{MacKay}. 

\begin{proof}
Consider the ADM approximation of the mutual information in Eq. \eqref{eq:objective.MI-ADM-KL}, augmented with regularization term $\lambda \sum_{k=1}^{K} \btheta_k^{T}\btheta_k$. 
Using this approximation, along with the expression of KL in \eqref{eq:KL}, it is easy to see that maximizing the regularized mutual information in \eqref{eq:RMI} can be stated as minimizing the following expression:  
\begin{align}
& \sum_{k=1}^{K}\hat{q}_{k} \log\Big(\hat{q}_{k} \Big)-\frac{1}{N}\sum_{i=1}^{N}\sum_{k=1}^{K}q_{ik} \log\big(p_{ik}\big) \nonumber \\
&+\frac{1}{N} \sum_{i=1}^{N}\sum_{k=1}^{K}q_{ik} \log\Big(\frac{q_{ik}}{p_{ik}}\Big) +\lambda \sum_{k=1}^{K} \btheta_k^{T}\btheta_k \label{eq:ML.Gener.1}\\
=& \sum_{k=1}^{K}\hat{q}_{k} \log\Big(\hat{q}_{k} \Big)+ \frac{1}{N}\sum_{i=1}^{N}\sum_{k=1}^{K}q_{ik} \log\big(q_{ik}\big) \nonumber \\
& - \frac{2}{N} \sum_{i=1}^{N}\sum_{k=1}^{K}q_{ik} \log\Big(\exp(\btheta_k^{T}\bz_i+b_k)\Big) \nonumber \\
&+\lambda \sum_{k=1}^{K} \btheta_k^{T}\btheta_k \label{eq:ML.Gener.2}\\
=&\sum_{k=1}^{K}\hat{q}_{k} \log\Big(\hat{q}_{k} \Big)+ \frac{1}{N}\sum_{i=1}^{N}\sum_{k=1}^{K}q_{ik}  \log\big(q_{ik}\big)  \nonumber \\
&- \frac{2}{N} \sum_{i=1}^{N}\sum_{k=1}^{K} q_{ik} b_k +\frac{1}{N} \Big[\sum_{i=1}^{N}\sum_{k=1}^{K}-2q_{ik} \btheta_k^{T} \bz_i \nonumber \\ 
& +N \lambda \sum_{k=1}^{K} \btheta_k^{T}\btheta_k\Big],\label{eq:ML.Gener.3}
\end{align}
where we replaced $p_{ik}$ by its expression in Eq. \eqref{eq:p_expression}. 
We recall that, in Eq. \eqref{eq:p_expression}, $b_k$ is the bias for cluster $k$.

Notice that the first term in the expression above is the negative entropy of the marginal distribution of labels. Minimization of this term prefers balanced partitions and, in fact, its global minimum is attained for a clustering verifying $\hat{q}_{k} = \frac{1}{K} ~ \forall k$.   
Now, assuming that the empirical label distribution is approximately uniform, \ie
$\hat{q}_{k} \approx \frac{1}{K}$, {we show in the Appendix that:}
\begin{align}\label{eq:approximation.2}
&\sum_{i=1}^{N}\sum_{k=1}^{K}-2q_{ik} \btheta_k^{T} \bz_i +N \lambda \sum_{k=1}^{K} \btheta_k^{T}\btheta_k\approx \nonumber \\
&~~\sum_{i=1}^{N}\sum_{k=1}^{K} q_{ik} \Arrowvert \frac{1}{\sqrt{\lambda K}} \bz_i- \sqrt{\lambda K} \btheta_k\Arrowvert^2 -\frac{1}{\lambda K} \sum_{i=1}^{N} \bz_i^{T} \bz_i.
\end{align}
Using \eqref{eq:approximation.2}, we obtain the following approximation of the regularized mutual information in \eqref{eq:ML.Gener.3}:  
\begin{align}
& \frac{1}{N}\sum_{i=1}^{N}\sum_{k=1}^{K}q_{ik} \log\hat{q}_{k} - \frac{2}{N} \sum_{i=1}^{N}\sum_{k=1}^{K} q_{ik} b_k \nonumber \\ 
& + \frac{1}{N}\sum_{i=1}^{N}\sum_{k=1}^{K}q_{ik}\log\big(q_{ik}\big)  -\frac{1}{N\lambda K} \sum_{i=1}^{N} \bz_i^{T} \bz_i  \nonumber \\ 
& +\frac{1}{N} \sum_{i=1}^{N}\sum_{k=1}^{K} q_{ik} \Arrowvert \frac{1}{\sqrt{\lambda K}} \bz_i- \sqrt{\lambda K} \btheta_k\Arrowvert^2 \label{eq:ML.Gener.2.3} \\
& = \mbox{KL}\Big(\hat{q}_{k} \Arrowvert \exp(2b_k)\Big) + \frac{1}{N}\sum_{i=1}^{N}\sum_{k=1}^{K}q_{ik} \log(q_{ik})  \nonumber \\
& -\frac{1}{N\lambda K} \sum_{i=1}^{N} \bz_i^{T} \bz_i +\frac{1}{N \lambda K} \sum_{i=1}^{N}\sum_{k=1}^{K} q_{ik} \Arrowvert  \bz_i- \btheta_k'\Arrowvert^2, \label{eq:ML.Gener.2.4}
\end{align}
where $\btheta_k'= \lambda K \btheta_k$. Notice that we re-wrote the first two terms in \eqref{eq:ML.Gener.2.3} in the form of a KL divergence. The convenience of this will soon become clear.
The optimization problem we obtained in \eqref{eq:ML.Gener.2.4} can be solved by alternating optimization w.r.t assignments $q_{i,k}$, parameters $\btheta_k'$ and biases $b_k$.  
Since $\mbox{KL}\Big(\hat{q}_{k}\Arrowvert \exp(2b_k)\Big) \ge 0$ and is equal to $0$ if and only the distributions are equal, the optimal $b_k$ can be expressed in closed-form as:
\begin{align}\label{eq:b_k.Gener}
\exp(2b_k)= \hat{q}_{k} ~\Longleftrightarrow ~ b_k= \frac{1}{2}\log\big(\hat{q}_{k}\big).
\end{align}
Substituting these optimal biases back into \eqref{eq:ML.Gener.2.4}, the KL term vanishes and \eqref{eq:ML.Gener.2.4} becomes equivalent to the regularized soft K-means in \eqref{eq:R-soft-kmeans}. 
\end{proof}

We refer to the soft and regularized K-means objective in \eqref{eq:R-soft-kmeans} as SR-K-means. Using this objective jointly with a reconstruction loss, the problem amounts to alternating optimization w.r.t $\bQ$, $\set{\btheta'_1,\ld,\btheta'_K}$ and network parameters $\mcW$.
Setting the partial derivatives of \eqref{eq:R-soft-kmeans} with respect to $\btheta_k'$ and $q_{ik}$ equals to zero, we obtain the corresponding optima 
in closed form as:
\begin{align}\label{eq:centroids.Gener}
\btheta'_k=\frac{\sum_{i=1}^{N}q_{ik}\bz_i}{ \sum_{i=1}^{N}q_{ik}},
\end{align}
and
\begin{align}\label{eq:qik.Gener}
q_{ik} \varpropto \exp \Big(- \frac{1}{\lambda K} \Big\Arrowvert \bz_i-\btheta'_k \Big\Arrowvert^2 \Big)
\end{align}
These updates clearly correspond to the well-known generative K-means algorithm. Eq. \eqref{eq:qik.Gener} uses a {\em softmin} function: it is a soft version of the standard hard (binary) assignments rule of K-means: $q_{ik}$ = 1 if $k = \argmin_{l} \Big\Arrowvert \bz_i-\btheta'_l \Big\Arrowvert$.    
Such soft K-means updates are known in the literature; see \cite[p. 289]{MacKay}. Also, Eq. \eqref{eq:centroids.Gener} is clearly a soft version of the mean updates in 
the standard K-means. Notice that, here, the $\btheta$-updates are in closed-form, unlike earlier for discriminative models DEPICT and MI-ADM, in which $\btheta$-updates
are performed within network training via stochastic gradient descent. It is also worth noting that balancing term  $\sum_{k=1}^{K}\hat{q}_{k} \log\Big(\hat{q}_{k} \Big)$ has
disappeared from our formulation in \eqref{eq:R-soft-kmeans} due to \eqref{eq:b_k.Gener}. This makes sense because it is well known that K-means has an implicit bias towards 
balanced partitions \cite{boykov2015volumetric}.

{
\section{Properties of ADM optimization for the mutual information}
\label{ADM-properties}

In this section, we prove that $\mcI(\mfX,\mfK)$ does not decrease at each two-step iteration of the ADM optimization in \eqref{eq:objective.MI-ADM-KL}, which can be viewed as a {\em bound maximization}\footnote{{Bound optimization, also known as MM (Minorize-Maximization) framework \cite{lange2000optimization, Zhang2007}, is a general principle, which updates the current solution to the next as the optimum of an auxiliary function (a tight lower bound on the original objective). This guarantees that the original objective function we want to maximize does not decrease at each iteration. The principle is widely used in machine learning as one trades a difficult optimization problem with a sequence of easier sub-problems \cite{Zhang2007}. Examples of well-known bound optimizers include expectation maximization (EM) algorithms, the concave-convex procedure (CCCP) \cite{Yuille2001} and submodular-supermodular procedures (SSP) \cite{Narasimhan2005}, among others.}} of the mutual information. We will also discuss the standard constrained-optimization approach based on multiplier-based quadratic penalties and its link to the KL penalties for simplex variables. 

\subsection{Bound-optimization interpretation}

Let $\mcI^{(t)}(\mfX,\mfK)$ denotes the MI at iteration $t$, \ie
\begin{align}
\mcI^{(t)}(\mfX,\mfK)  \triangleq \frac{1}{N} \sum_{i=1}^{N}\sum_{k=1}^{K}{p}_{ik}^{(t)} \log(p_{ik}^{(t)}) - \sum_{k=1}^{K}\hat{p}_{k}^{(t)}\log\Big(\hat{p}_{k}^{(t)} \Big).
\end{align}

\begin{proposition}\label{lemma.J}
$\mcI^{(t)}(\mfX,\mfK)$ is a non-deceasing function of iteration counter $t$ for the two-step 
mixed-variable ADM optimization in \eqref{eq:objective.MI-ADM-KL}.
\end{proposition}
\begin{proof}

Given the network parameter at time $t$ or, equivalently, the resulting posteriors $\bP^{(t)}$, the target estimation step at time $t$ explicitly implies:
\begin{align}\label{eq:objective.MI-ADM-KL-4}
& \frac{1}{N} \sum_{i=1}^{N}\sum_{k=1}^{K}{q}_{ik}^{(t)} \log(p_{ik}^{(t)}) - \sum_{k=1}^{K}\hat{q}_{k}^{(t)}\log\Big(\hat{q}_{k}^{(t)} \Big) -  \mbox{KL}(\bQ^{(t)}\Arrowvert \bP^{(t)}) \nonumber \\
& \ge \frac{1}{N} \sum_{i=1}^{N}\sum_{k=1}^{K}{q}_{ik} \log(p_{ik}^{(t)}) - \sum_{k=1}^{K}\hat{q}_{k}\log\Big(\hat{q}_{k} \Big) -  \mbox{KL}(\bQ\Arrowvert \bP^{(t)}), 
\end{align}
for all $\bQ$. Applying \eqref{eq:objective.MI-ADM-KL-4} to $\bQ=\bP^{(t)}$, we have the following upper bound on the mutual information at iteration $t$:
\begin{align}\label{eq:time.t}
& \frac{1}{N} \sum_{i=1}^{N}\sum_{k=1}^{K}{q}_{ik}^{(t)} \log(p_{ik}^{(t)}) - \sum_{k=1}^{K}\hat{q}_{k}^{(t)}\log\Big(\hat{q}_{k}^{(t)} \Big) -  \mbox{KL}(\bQ^{(t)}\Arrowvert \bP^{(t)}) \nonumber \\
& \ge \mcI^{(t)}(\mfX,\mfK).
\end{align}
Using the estimated $\bQ^{(t)}$, the parameter-learning step at time $t+1$ implies:
\begin{align}\label{eq:objective.MI-ADM-KL-5}
& \frac{1}{N} \sum_{i=1}^{N}\sum_{k=1}^{K}{q}_{ik}^{(t)} \log(p_{ik}^{(t+1)}) - \sum_{k=1}^{K}\hat{q}_{k}^{(t)}\log\Big(\hat{q}_{k}^{(t)} \Big)
\nonumber \\
&-  \mbox{KL}(\bQ^{(t)}\Arrowvert \bP^{(t+1)}) \nonumber \\
& \ge \frac{1}{N} \sum_{i=1}^{N}\sum_{k=1}^{K}{q}_{ik}^{(t)} \log(p_{ik}) - \sum_{k=1}^{K}\hat{q}_{k}^{(t)}\log\Big(\hat{q}_{k}^{(t)} \Big) 
 \nonumber  \\ 
 &-  \mbox{KL}(\bQ^{(t)}\Arrowvert \bP), 
\end{align}
for all $\bP$. Applying inequality \eqref{eq:objective.MI-ADM-KL-5} to $\bP=\bP^{(t)}$, and combining the result with inequality \eqref{eq:time.t}, we 
obtain the following upper bound on the mutual information at iteration $t$: 
\begin{align}\label{eq:time.p.t+1}
& \frac{1}{N} \sum_{i=1}^{N}\sum_{k=1}^{K}{q}_{ik}^{(t)} \log(p_{ik}^{(t+1)}) - \sum_{k=1}^{K}\hat{q}_{k}^{(t)}\log\Big(\hat{q}_{k}^{(t)} \Big)  \nonumber \\
& -  \mbox{KL}(\bQ^{(t)}\Arrowvert \bP^{(t+1)}) \ge \mcI^{(t)}(\mfX,\mfK).
\end{align}
Finally, using the fact that $ \bP^{(t+1)}=\bQ^{(t)}$, \eqref{eq:time.p.t+1} can be rewritten as:
\begin{align}\label{eq:time.t+1}
\mcI^{(t+1)}(\mfX,\mfK) \ge \mcI^{(t)}(\mfX,\mfK), 
\end{align}
which terminates the proof.

\end{proof}

{It is worth noting that the generative SR-Kmeans procedure discussed earlier can also be viewed as a bound optimizer for the mutual information in Eq. \eqref{eq:RMI}, but it uses a bound (auxiliary function) different from the discriminative ADM procedure. In the discriminative ADM procedure, we optimizes directly Eq. \eqref{eq:ML.Gener.1}, with the assignment-variable updates derived from setting the approximate gradient of Eq. \eqref{eq:ML.Gener.1} with respect to assignment variables $q_{ik}$ equal to zero. In the generative SR-Kmeans procedure, we still optimize \eqref{eq:ML.Gener.1} with respect to $q_{ik}$, but in an {\em indirect} way using the equivalent K-means objective in Eq. \eqref{eq:R-soft-kmeans} and a two-step process: one step updating cluster prototypes (means) with Eq. (19) and the other updating assignment variables. In fact, for the standard K-means procedure, one can show that this two-step process with prototype updates is a bound optimization\footnote{{The prototype updates correspond to building a bound (auxiliary function) on high-order K-means objective expressed solely as a function of the assignment variables (i.e., the prototypes in the K-means objective are expressed with assignment variables).}}; See, for instance, Theorem 1 in \cite{kernelcut}. Therefore, as optimizing the ADM version in Eq. \eqref{eq:ML.Gener.1} is a bound optimizer for the mutual information (Proposition \ref{lemma.J}), the SR-Kmeans procedure can also be viewed as a bound optimizer for the mutual information, but with a different auxiliary function, as it uses a prototype-based bound on Eq. \eqref{eq:ML.Gener.1}.}

\subsection{Other optimization alternatives and the link between KL and quadratic penalties for simplex variables}

In this section, we discuss other alternatives for solving constrained optimization problem \eqref{eq:objective.MI-ADM}. In fact, the standard alternating 
direction method of multipliers (ADMM) method \cite{Boyd_EE364b} solves \eqref{eq:objective.MI-ADM} iteratively via the following updates:
\begin{align}
\Phi^{(t+1)}&= \argmax_{\Phi} ~ L\big(\Phi,\bQ^{(t)},\{\lambda_{i,k}^{(t)}\}_{\tiny{1\le i\le N}}^{\tiny{1\le k\le K}},\rho \big), \label{eq:ADMM.p} \\
\bQ^{(t+1)}&= \argmax_{\bQ} ~ L\big(\Phi^{(t+1)},\bQ,\{\lambda_{i,k}^{(t)}\}_{\tiny{1\le i\le N}}^{\tiny{1\le k\le K}},\rho\big), \label{eq:ADMM.q}\\
\lambda_{i,k}^{(t+1)} &= \lambda_{i,k}^{(t)} + \rho \cd  (q_{ik}-p_{ik}), \quad {\tiny{1\le i\le N}}, \quad {\tiny{1\le k\le K}}, \label{eq:ADMM.multiplier}
\end{align}
where $\rho>$ is called the penalty parameter, $\{\lambda_{i,k}\}_{\tiny{1\le i\le N}}^{\tiny{1\le k\le K}}$ are the Lagrange multipliers and $L$ is the augmented Lagrangian function defined as:
\begin{align}
\label{eq:augmented-Lagrangian}
&L\big(\Phi,\bQ,\{\lambda_{i,k}\}_{\tiny{1\le i\le N}}^{\tiny{1\le k\le K}},\rho\big) = \frac{1}{N} \sum_{i=1}^{N}\sum_{k=1}^{K}{q}_{ik} \log(p_{ik}) \nonumber \\
&~~ - \sum_{k=1}^{K}\hat{q}_{k}\log\Big(\hat{q}_{k} \Big) - \sum_{i=1}^{N}\sum_{k=1}^{K} \lambda_{i,k} \big(q_{ik}-p_{ik}\big) -\frac{\rho}{2}\Arrowvert \bQ-\bP\Arrowvert ^2_2
\end{align}
In \eqref{eq:augmented-Lagrangian}, equality constraint $\bQ = \bP$ is handled via a multiplier-based quadratic penalty, i.e., the last two terms in \eqref{eq:augmented-Lagrangian}. The use of the KL penalty has important computational advantages over this standard augmented Lagrangian formulation. First, 
it is not straightforward to solve \eqref{eq:ADMM.q} analytically by setting to zero the partial derivative of $L$, which is given by: 
\begin{align}
\label{eq:augmented-Lagrangian-NC}
& \frac{\partial L\big(\Phi^{(t+1)},\bQ,\{\lambda_{i,k}^{(t)}\}_{\tiny{1\le i\le N}}^{\tiny{1\le k\le K}},\rho\big) }{ \partial q_{ik}} =  \frac{1}{N} \log\big(p_{ik}^{(t+1)}\big) \nonumber \\
&~~~ -\log\big( \frac{1}{N} \sum_{i=1}^{N} {q}_{ik} \big) -  \frac{1}{N} - \lambda_{i,k}^{(t)}-\rho \cd (q_{ik}-p_{ik}^{(t+1)})
\end{align}
Here, a numerical method, with additional inner iterations, might be needed. We can use a faster, penalty-based version of \eqref{eq:augmented-Lagrangian} by removing the Lagrange-multiplier term (third term). This removes point-wise multiplier updates \eqref{eq:ADMM.multiplier}, but the quadratic penalty would still require inner iterations for solving \eqref{eq:ADMM.q}.   
Second, while quadratic penalties are more standard in the general context of constrained optimization, the KL penalty we have in \eqref{eq:objective.MI-ADM-KL} has important computational advantages in the case of simplex constraints. In fact, the KL penalty 
in \eqref{eq:objective.MI-ADM-KL} has a negative-entropy barrier term, which completely removes extra 
Lagrangian-dual iterations/projections to handle simplex constraints $\bq_i^t {\mathbf 1} = 1$ and $\bq_i \geq 0 ~ \forall i$. 
Such a barrier forces each assignment variable to be non-negative, which removes the need for extra dual variables for 
constraints $\bq_i \geq 0$, and conveniently yields closed-form updates for the dual variables of constraints $\bq_i^t {\mathbf 1} = 1$.
These computational advantages over quadratic penalties are important, more so when dealing with large data sets. It is worth noting that, for dealing with simplex constraints, KL-based penalties are common in the context of Bregman-proximal optimization \cite{Yuan2017}, with established computational and memory advantages over quadratic penalties \cite{Yuan2017}. However, to our knowledge, they are less common in the clustering literature.

Moreover, for simplex variables, there is an interesting link between KL and quadratic penalties, which comes directly from the {\em Pinsker's inequality} \cite{csiszar_korner_2011}. In fact, 
For any $\bP$ and $\bQ$ containing probability simplex vectors, Pinsker's inequality states that the quadratic penalty is upper-bounded by KL (up to a multiplicative constant):  
\begin{equation}
\Arrowvert \bQ-\bP\Arrowvert ^2_2 \leq 2 \, \mbox{KL}(\bQ \Arrowvert \bP)
\end{equation}
Therefore, for simplex variables, minimizing KL corresponds also to minimizing an upper bound on the quadratic penalty.  
}

\section{Experiments}\label{sec.experiments}
\subsection{Reconstruction loss and {implementation details}}

We adopted the reconstruction loss and DNN architecture proposed recently in \cite{Dizaji17} for our experiments. 
The architecture consists of a multi-layer convolutional denoising auto-encoder with stridded convolutional layers in the decoder part. It is composed of three components:
\begin{itemize}
\item
A corrupted encoder, which maps the noisy input into the embedding space. The output of each noisy encoder layer is given by:
\begin{align}
\hat{\bz}^{l}=Dropout\Big(g\big(\W^{l}_\tr{e} \hat{\bz}^{l-1}\big)\Big),
\end{align}
where $Dropout(\cd)$ is a stochastic mapping that randomly sets a portion of its inputs to zero \cite{Srivastava14}, $g$ is the activation function 
and $\W^{l}_\tr{e}$ denotes the weights of the $l$-th encoder. $L$ denotes the depth of the auto-encoder.
\item A clean decoder, which follows the corrupted encoder. The reconstruction of each layer is defined as:
\begin{align}
\tilde{\bz}^{l-1}=g\Big(\W^{l}_\tr{d} \tilde{\bz}^{l-1}\Big)
\end{align}
where $\W^{l}_\tr{d}$ are the weights of the $l$-th decoder layer.
\item A clean encoder, which has the same weights as the corrupted one, \ie the output of the $l$-th layer is expressed as:
\begin{align}
\bz^{l}=\W^{l}_\tr{e} (\bz^{l-1})
\end{align}
\end{itemize}
We used the rectified linear units (ReLUs) \cite{Nair10} as activation functions. 
For further details on the architecture, refer to \cite[Sec. 3.2] {Dizaji17}. We note that the adopted architecture is similar to the Ladder network \cite{Valpola15}, where the clean pathway is used for prediction while the corrupted one guaranties that the network is noise-invariant.

As in \cite{Dizaji17}, and in order to avoid over-fitting, we add a reconstruction loss function to our objectives MI-ADM in Eq. \eqref{eq:objective.MI-ADM-KL-2} and 
in Eq. \eqref{eq:R-soft-kmeans}: 
\begin{align}
{\cal R}(\mcZ) =   \frac{1}{N} \sum_{i=1}^{N}\sum_{l=0}^{L-1} \frac{1}{|\bz_i^l|} \Arrowvert \bz_i^l-\tilde{\bz}_i^l \Arrowvert^2,
\end{align}
where $|\bz_i^l|$ is the output size of the $l$-th layer.
In the experiments described below, MI-ADM refers to the process that alternates the target updates of Eq. \eqref{eq:optimal.q.Disc} with learning network parameters that 
optimize the following loss:
\begin{align}\label{eq:maximization.Disc.2}
\min_{\Phi}~-\frac{1}{N}\sum_{i=1}^{N}\sum_{k=1}^{K}q_{ik} \log p_{ik} + {\cal R}(\mcZ)  
\end{align}
SR-K-means refer to the process that alternates the soft K-means updates in Eqs. \eqref{eq:centroids.Gener} and \eqref{eq:qik.Gener} and learning network parameters that 
optimize the following loss:   
\begin{align}\label{eq:maximization.Gener.3}
\min_{\mcW}~\frac{1}{N \lambda K} \sum_{i=1}^{N}\sum_{k=1}^{K} q_{ik} \Arrowvert \bz_i- \btheta_k'\Arrowvert^2-\frac{1}{N\lambda K} \sum_{i=1}^{N} \bz_i^{T} \bz_i 
+ {\cal R}(\mcZ) 
\end{align}
Note that the deep network parameters $\mcW$ are firstly initialized considering the auto-encoder only \cite{Dizaji17,Xie16,glorot10a}: $\min_{\Phi}~{\cal R}(\mcZ)$. 
Then the initial features are clustered via soft K-means to obtain the initial targets $\bQ$. {Regarding the optimization method and hyper-parameter selection, we use the ones adopted in \cite{Dizaji17}, except for new regularization parameter $\lambda$, which we introduced in \eqref{eq:RMI}. 
Namely, as stochastic optimizer, we adopt Adam \cite{kingma2014adam} with default parameters $\beta_1=0.9$, $\beta_2=0.999$ and $\epsilon=1\e^{-8}$.
We initialize weights using Xavier approach \cite{glorot10a}. The mini-batch size, learning rate and dropout parameter are set to $100$, $1\e^{-3}$ and $0.1$, respectively. Regarding $\lambda$, as models \eqref{eq:MI} and \eqref{eq:RMI} become equivalent when $\lambda \rightarrow 0$, we tested the values $10^{i}$, $i=-1,\ld,-5$. We found that $\lambda = 10^{-4}$ yields the best performance for model \eqref{eq:RMI}.   
}

\subsection{Results}

\subsubsection{Data sets}

In order to confirm the theoretical link in Proposition \ref{prop:link-to-K-means} between discriminative model MI-ADM in \eqref{eq:objective.MI-ADM-KL-2} and generative model 
SR-K-means in \eqref{eq:R-soft-kmeans}, we evaluated them on two handwriting datasets (USPS and MNIST) and three face datasets (Youtube-Face, CMU-PIE and FRGG \cite{Yang16}). Table~\ref{tab.datasets} presents a summary of the statistics of the data sets.

\begin{table*}[t]
  \caption{Description of the datasets}
  \label{tab.datasets}
  \centering
  \begin{tabular}{lllllll}
\toprule
   Dataset   & \# Samples & \# Classes & \#Dimensions & \% of smallest class & \% of largest  class\\
\midrule
  USPS  &  $11000$ & $10$ & $1\times 16\times 16$ & 10 \% & 10 \% \\
MNIST-test     & $10000$ & $10$ & $1\times 28\times 28$ & 8.92 \% & 11.35 \% \\
 MNIST-full     & $79000$ & $10$ & $1\times 28\times 28$ & 9.01 \% & 11.25 \% \\
 Youtube-Face (YTF)     & $10000$ & $41$ & $3\times 55\times 55$ & 0.31 \% & 6.94 \%\\
 CMU-PIE  & $2856$ & $68$ & $1\times 32\times 32$ & 1.47 \% & 1.47 \%  \\
 FRGC  & $2462$ & $20$ & $3\times 32\times 32$ &  0.24\% & 10.51 \%  \\
\bottomrule
  \end{tabular}
\end{table*}
\subsubsection{Performance metrics}
We adopt two standard unsupervised evaluation metrics: the accuracy (ACC) and the normalized mutual information (NMI). ACC captures the best matching between the unsupervised clustering results and the ground truth \cite{Kuhn55}. NMI translates the similarity between pairs of clusters, and is invariant w.r.t permutations \cite{Xu03}.

\subsubsection{Evaluation of clustering algorithms}

Table~\ref{tab.datasets} reports the results\footnote{Our code is publicly available at: \\ \url{https://github.com/MOhammedJAbi/SoftKMeans}} of discriminative model MI-ADM in \eqref{eq:objective.MI-ADM-KL-2} and generative model SR-K-means in \eqref{eq:R-soft-kmeans}. We also include the results of several related models: (1) the DEPICT model \cite{Dizaji17} based on KL and logistic regression posteriors, which achieves a state-of-the-art performance on MNIST; (2) DEC \cite{Xie16}, also a KL-based approach assuming $t$-distribution between embedded points and cluster prototypes; and DCN \cite{Yang17}, which optimizes a loss containing a hard K-means term and a reconstruction term.

{The numerical results show that MI-ADM and SR-K-means algorithms may yield comparable results even for unbalanced data sets, \eg YTF. We recall here that our analysis was done assuming the clusters are balanced.} Also, notice that MI-ADM and DEPICT have approximately the same performance, confirming our earlier discussion: MI-ADM in \eqref{eq:objective.MI-ADM-KL-2} can be viewed as an approximation of DEPICT in \eqref{eq:kl.disc}. The additional entropy term in \eqref{eq:objective.MI-ADM-KL-2}, ${\mbox H}(\bQ)$, has almost no effect on the results. Finally, notice the substantial difference in performance ($11\%$) between our regularized and soft K-means and DCN \cite{Yang17}, which is based on a hard K-means loss. 

\begin{table*}[t]
  \caption{Comparison of clustering algorithms on four date sets based on accuracy and normalized mutual information. The results of SR-K-means algorithm are obtained using $\lambda=10^{-4}$. (-), (*), ($^\dag$) stands for ``not reported'' and ``reported" in \cite{Yang17} and\cite{Dizaji17}, respectively.}
  \label{tab:results}
  \centering
  \begin{tabular}{lllllllllllll}
 \toprule
  Dataset     &  \multicolumn{2}{c}{USPS}     &  \multicolumn{2}{c}{MNIST-test } &  \multicolumn{2}{c}{MNIST-full } & \multicolumn{2}{c}{YTF }  & \multicolumn{2}{c}{CMU-PIE } & \multicolumn{2}{c}{FRGC}\\
\midrule
& NMI & ACC & NMI & ACC & NMI & ACC & NMI & ACC & NMI & ACC & NMI & ACC  \\
\midrule
MI-ADM & 0.948 & 0.979 &  0.885 & 0.871 & 0.922 & 0.969 & 0.801 & 0.606 & 0.965 & 0.858 & 0.580 & 0.431\\
SR-K-means & 0.936 & 0.974 &  0.873 & 0.863 & 0.866 & 0.939 & 0.806 & 0.605 & 0.945 & 0.902 & 0.487 & 0.413\\
DEPICT \cite{Dizaji17} & 0.945 & 0.978 &  0.886 & 0.872 & 0.925 & 0.971 &0.802 & 0.611 & 0.964 & 0.850 & 0.583 & 0.432 \\
DCN  (K-means based) \cite{Yang17}& - & - &  - & - & 0.81$^{*}$ & 0.83$^{*}$ & - & - & - & - & - & -\\
DEC (KL based) \cite{Xie16} & 0.586$^\dag$ & 0.619$^\dag$ & 0.827$^\dag$ & 0.859$^\dag$ & 0.816$^\dag$ & 0.844$^\dag$ & 0.446$^\dag$ & 0.371$^\dag$ & 0.924$^\dag$ & 0.801$^\dag$ & 0.505$^\dag$ & 0.378$^\dag$\\
\bottomrule
  \end{tabular}
\end{table*}

{


\begin{table*}[t]
  \caption{{Comparison of MI-ADM and MI-D on four datasets (accuracy and normalized mutual information).}} 
  \label{tab:results.ADM.adv}
  \centering
  \begin{tabular}{lllllllllllll}
 \toprule
 { Dataset}     &  \multicolumn{2}{c}{{USPS}}     &  \multicolumn{2}{c}{{MNIST-test}} &  \multicolumn{2}{c}{{MNIST-full}} & \multicolumn{2}{c}{{YTF}}  & \multicolumn{2}{c}{{CMU-PIE}} & \multicolumn{2}{c}{{FRGC}}\\
\midrule
& {NMI} & {ACC} & {NMI} & {ACC} & {NMI} & {ACC} & {NMI} & {ACC} & {NMI} & {ACC} & {NMI} & {ACC}  \\
\midrule
{MI-ADM} & {0.948} & {0.979} &  {0.885} & {0.871} & {0.922} & {0.969} & {0.801} & {0.606} & {0.965} & {0.858} & {0.580} & {0.431} \\
{MI-D} & {0.948} & {0.979} &  {0.880} & {0.867} & {0.921} & {0.967} & {0.800} & {0.616} & {0.867} & {0.705} & {0.444} & {0.322} 
\\
\bottomrule
  \end{tabular}
\end{table*}

\subsubsection{Direct maximization of the mutual information}


In this section, we report the clustering results when mutual information $\mcI(\mfX,\mfK)$ is maximized directly,  as proposed in \cite{Krause10,Hu17,Bridle92}, which is different from the two-step ADM approach in \eqref{eq:q.DEPICT} \& \eqref{eq:param-learn-depict}. This amounts to solving directly the following optimization problem via SGD:
$\max_{\Phi}~ \mcI(\mfX,\mfK)$.
We refer to such direct maximization as MI-D, and report the results\footnote{{Again, here, and as done for MI-ADM and SR-K-means, we added the reconstruction term when learning the network parameters, \ie we solve $\min_{\Phi}-\mcI(\mfX,\mfK) + {\cal R}(\mcZ)$ in MI-D.}} in Table \eqref{tab:results.ADM.adv}. 
As we can see, while the results of the two algorithms are practically the same for the majority of considered datasets, the MI-ADM optimizer outperforms MI-D on FRGC and CMU-PIE in terms of performance measures NMI and ACC.  
Regarding optimization performance, MI-ADM converges to a local optimum that is ``better'' than the one obtained with MI-D on the FRGC data set, but this is not the case for MNIST, for instance, where both methods converge approximately to the same solution. Our results may not be enough to claim that MI-ADM is a better optimizer than MI-D, in general, as none of the two optimization methods (MI-ADM and MI-D) provides optimality guarantee or bounds for highly non-convex problems involving deep networks. However, our results are consistent with recent optimization works, \eg \cite{Taylor2016,Wang2019}, which showed that variants of ADMM could be effective alternatives to SGD for supervised deep learning problems. Those differences in optimization performances do not provide a full explanation of the fact that MI-ADM outperforms MI-D on FRGC and CMU-PIE. In fact, to the best of our knowledge, there is no rigorous quantification of how maximizing the mutual information, \ie $\mcI(\mfX,\mfK)$, which is computed based on the predictions, is maximizing the commonly used performance metrics, \ie the  accuracy (ACC) and normalized  mutual  information (NMI), which are both computed based on the true ground-truth labels, for unsupervised clustering problems. 
 
Finally, notice that \figref{fig.MI.comparison} confirms Proposition \ref{lemma.J} experimentally, i.e. the MI does not increase during the iterations of KL-based ADM, similarly to a direct SGD maximization. We used the FRGC data set for this plot as a typical example, but the MI evolution during the iterations of MI-ADM follows the same form for the remaining data sets.



\begin{figure}[tb] 

\psfrag{iteration}[cc][cc][0.5]{$\mcI(\mfX,\mfK)$}

\begin{center}

\scalebox{1}{\includegraphics[width=0.95\linewidth]{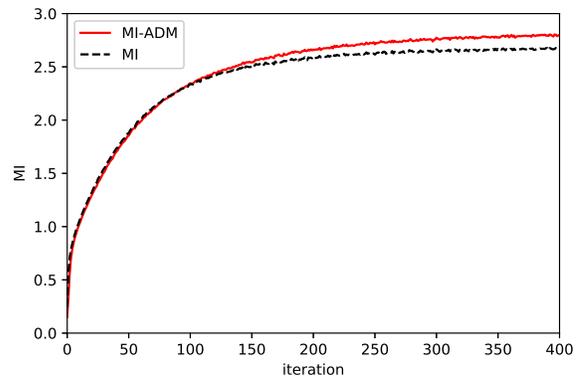}}

\caption{{Evolution of the MI, $\mcI(\mfX,\mfK)$, during the iterations of MI-ADM and MI-D algorithms for the FRGC data set.}}\label{fig.MI.comparison}
\end{center}
\end{figure}
}

\section{Conclusion}\label{sec.conclusion}

We showed that several prevalent state-of-the-art models for deep clustering are equivalent to K-means under mild conditions and commonly used posterior models and parameter regularization. We proved that, for the standard logistic regression posteriors, maximizing the $L_2$ regularized mutual information via the alternating direction method (ADM) is equivalent to a soft and regularized K-means loss. Our theoretical analysis not only connected directly several recent discriminative models to K-means, but also led to a new soft and regularized deep K-means algorithm, which gave competitive results on several image clustering benchmarks. Furthermore, our result suggests several interesting extensions for future works. For instance, it is well known that simple parametric prototypes such as the means, as in K-means, may not be good representatives of manifold-structured and high-dimensional inputs such as images. Investigating other prototype-based objectives such as K-modes \cite{Ziko2018} may provide better representatives of the data. Also, for manifold-structured inputs, investigating pairwise clustering objectives such as normalized cut \cite{Shaham18}, in conjunction with reconstruction losses, might be more appropriate for deep image clustering. Namely, it is interesting to see how using the loss function of the aforementioned algorithms, \eg \cite[Eq. (1)] {Ziko2018} and \cite[Eq. (3)] {Shaham18}, instead of the K-means components in \eqref{eq:maximization.Gener.3}, will affect the results. Also, it is worthy to investigate a possible link between MI and the loss functions of those algorithms.

As a final comment, we add here the KL divergence to enforce constraint $\bQ=\bP$ when going from \eqref{eq:objective.MI-ADM} to \eqref{eq:objective.MI-ADM-KL}. As a a future work, it will be interesting to analyze the optimality and convergence of MI-ADM if different distance measures, \eg Bhattacharyya measures family \cite{Kailath67}, and penalty methods for constrained optimization, \eg generalized quadratic penalty \cite{BERTSEKAS76}, are adopted.


\begin{appendices}
\section{}
{ This appendix derives the approximation in \eqref{eq:approximation.2}. Assuming that the empirical label distribution is approximately uniform, \ie$\hat{q}_{k} = \frac{1}{N}\sum_{i=1}^{N} q_{ik}\approx  \frac{1}{K}  \iff \frac{K}{N}\sum_{i=1}^{N} q_{ik}\approx 1 ~~\forall k $, we have
\begin{align}
&\sum_{i=1}^{N}\sum_{k=1}^{K}-2q_{ik} \btheta_k^{T} \bz_i +N \lambda \sum_{k=1}^{K} \btheta_k^{T}\btheta_k \\
&~~ \approx \sum_{i=1}^{N}\sum_{k=1}^{K}-2q_{ik} \btheta_k^{T} \bz_i +N \lambda \sum_{k=1}^{K} \Big(\frac{K}{N}\sum_{i=1}^{N} q_{ik}\Big) \btheta_k^{T}\btheta_k \\
&~~ = \sum_{i=1}^{N}\sum_{k=1}^{K}-2q_{ik} \btheta_k^{T} \bz_i +K \lambda \sum_{k=1}^{K} \sum_{i=1}^{N} q_{ik} \btheta_k^{T}\btheta_k \\
&~~ = \sum_{i=1}^{N}\sum_{k=1}^{K}-2q_{ik} \btheta_k^{T} \bz_i +K \lambda \sum_{i=1}^{N} \sum_{k=1}^{K}  q_{ik} \btheta_k^{T}\btheta_k \nonumber \\
&~~~ +\frac{1}{\lambda K} \sum_{i=1}^{N} \bz_i^{T} \bz_i -\frac{1}{\lambda K} \sum_{i=1}^{N} \bz_i^{T} \bz_i \\
&~~ = \sum_{i=1}^{N}\sum_{k=1}^{K}-2q_{ik} \btheta_k^{T} \bz_i +K \lambda \sum_{i=1}^{N} \sum_{k=1}^{K}  q_{ik} \btheta_k^{T}\btheta_k \nonumber \\
&~~~ +\frac{1}{\lambda K} \sum_{i=1}^{N} \Big(\sum_{k=1}^{K} q_{ik}\Big)  \bz_i^{T} \bz_i -\frac{1}{\lambda K} \sum_{i=1}^{N} \bz_i^{T} \bz_i \\
&~~ = \sum_{i=1}^{N}\sum_{k=1}^{K}q_{ik} \Big (-2\btheta_k^{T} \bz_i +\big(\sqrt{K \lambda}\big)^2\btheta_k^{T}\btheta_k + \big(\frac{1}{\sqrt{\lambda K}}  \big)^2 \bz_i^{T} \bz_i  \Big)\nonumber \\
&~~~  -\frac{1}{\lambda K} \sum_{i=1}^{N} \bz_i^{T} \bz_i \\
&~~ = \sum_{i=1}^{N}\sum_{k=1}^{K} q_{ik} \Arrowvert \frac{1}{\sqrt{\lambda K}} \bz_i- \sqrt{\lambda K} \btheta_k\Arrowvert^2 -\frac{1}{\lambda K} \sum_{i=1}^{N} \bz_i^{T} \bz_i.
\end{align}
}
\end{appendices}

\bibliography{biblio}

\begin{thebibliography}{10}
\providecommand{\url}[1]{#1}
\csname url@samestyle\endcsname
\providecommand{\newblock}{\relax}
\providecommand{\bibinfo}[2]{#2}
\providecommand{\BIBentrySTDinterwordspacing}{\spaceskip=0pt\relax}
\providecommand{\BIBentryALTinterwordstretchfactor}{4}
\providecommand{\BIBentryALTinterwordspacing}{\spaceskip=\fontdimen2\font plus
\BIBentryALTinterwordstretchfactor\fontdimen3\font minus
  \fontdimen4\font\relax}
\providecommand{\BIBforeignlanguage}[2]{{%
\expandafter\ifx\csname l@#1\endcsname\relax
\typeout{** WARNING: IEEEtran.bst: No hyphenation pattern has been}%
\typeout{** loaded for the language `#1'. Using the pattern for}%
\typeout{** the default language instead.}%
\else
\language=\csname l@#1\endcsname
\fi
#2}}
\providecommand{\BIBdecl}{\relax}
\BIBdecl

\bibitem{Trigeorgis2014}
G.~Trigeorgis, K.~Bousmalis, S.~Zafeiriou, and B.~Schuller, ``A deep semi-nmf
  model for learning hidden representations,'' in \emph{International
  Conference on Machine Learning (ICML)}, 2014, pp. 1692--1700.

\bibitem{Caron2018}
M.~Caron, P.~Bojanowski, A.~Joulin, and D.~Matthijs, ``Deep clustering for
  unsupervised learning of visual features,'' in \emph{European Conference On
  Computer Vision (ECCV)}, 2018, pp. 1692--1700.

\bibitem{Yang17}
B.~Yang, X.~Fu, N.~D. Sidiropoulos, and M.~Hong, ``Towards k-means-friendly
  spaces: Simultaneous deep learning and clustering,'' in \emph{International
  Conference on Machine Learning (ICML)}, 2017, pp. 3861--3870.

\bibitem{Dizaji17}
K.~Ghasedi~Dizaji, A.~Herandi, C.~Deng, W.~Cai, and H.~Huang, ``Deep clustering
  via joint convolutional autoencoder embedding and relative entropy
  minimization,'' in \emph{International Conference on Computer Vision (ICCV)},
  2017, pp. 5747--5756.

\bibitem{Hu17}
W.~Hu, T.~Miyato, S.~Tokui, E.~Matsumoto, and M.~Sugiyama, ``Learning discrete
  representations via information maximizing self-augmented training,'' in
  \emph{International Conference on Machine Learning (ICML)}, 2017, pp.
  1558--1567.

\bibitem{Jiang16}
Z.~Jiang, Y.~Zheng, H.~Tan, B.~Tang, and H.~Zhou, ``Variational deep embedding:
  A generative approach to clustering,'' in \emph{International Joint
  Conference on Artificial Intelligence (IJCAI)}, 2017, pp. 1965--1972.

\bibitem{Springenberg15}
J.~T. Springenberg, ``Unsupervised and semi-supervised learning with
  categorical generative adversarial networks,'' in \emph{International
  Conference on Learning Representations (ICLR)}, 2016, pp. 1965--1972.

\bibitem{Yang16}
J.~Yang, D.~Parikh, and D.~Batra, ``Joint unsupervised learning of deep
  representations and image clusters,'' in \emph{IEEE Conference on Computer
  Vision and Pattern Recognition (CVPR)}, 2016, pp. 5147--5156.

\bibitem{Xie16}
J.~Xie, R.~Girshick, and A.~Farhadi, ``Unsupervised deep embedding for
  clustering analysis,'' in \emph{International Conference on Machine Learning
  (ICML)}, 2016, pp. 478--487.

\bibitem{Biernacki00}
C.~Biernacki, G.~Celeux, and G.~Govaert, ``Assessing a mixture model for
  clustering with the integrated completed likelihood,'' \emph{{IEEE}
  Transactions on Pattern Analysis and Machine Intelligence}, vol.~22, no.~7,
  pp. 719--725, 2000.

\bibitem{ShiMalik2000}
J.~Shi and J.~Malik, ``Normalized cuts and image segmentation,'' \emph{IEEE
  Transactions on Pattern Analysis and Machine Intelligence}, vol.~22, no.~8,
  pp. 888--905, 2000.

\bibitem{kernelcut}
M.~Tang, D.~Marin, I.~Ben~Ayed, and Y.~Boykov, ``Kernel cuts: Kernel and
  spectral clustering meet regularization,'' \emph{International Journal of
  Computer Vision}, 2019.

\bibitem{DensityBias}
D.~Marin, M.~Tang, I.~Ben~Ayed, and Y.~Boykov, ``Kernel clustering: density
  biases and solutions,'' \emph{IEEE Transactions on Pattern Analysis and
  Machine Intelligence}, 2018.

\bibitem{Krause10}
A.~Krause, P.~Perona, and R.~G. Gomes, ``Discriminative clustering by
  regularized information maximization,'' in \emph{Neural Information
  Processing Systems (NIPS)}, 2010, pp. 775--783.

\bibitem{Bridle92}
J.~S. Bridle, A.~J.~R. Heading, and D.~J.~C. MacKay, ``Unsupervised
  classifiers, mutual information and ``phantom targets'','' in \emph{Neural
  Information Processing Systems (NIPS)}, 1992, pp. 1096--1101.

\bibitem{Grandvalet2004}
Y.~Grandvalet and Y.~Bengio, ``Semi-supervised learning by entropy
  minimization,'' in \emph{Neural Information Processing Systems (NIPS)}, 2004,
  pp. 529--536.

\bibitem{boyd2011distributed}
S.~Boyd, N.~Parikh, E.~Chu, B.~Peleato, and J.~Eckstein, ``Distributed
  optimization and statistical learning via the alternating direction method of
  multipliers,'' \emph{Foundations and Trends{\textregistered} in Machine
  Learning}, vol.~3, no.~1, pp. 1--122, 2011.

\bibitem{Aljalbout18}
E.~Aljalbout, V.~Golkov, Y.~Siddiqui, and D.~Cremers, ``Clustering with deep
  learning: Taxonomy and new methods,'' \emph{arXiv: 1801.07648}, 2018.

\bibitem{MacKay}
D.~J.~C. MacKay, \emph{Information Theory, Inference and Learning
  Algorithms}.\hskip 1em plus 0.5em minus 0.4em\relax Cambridge University
  Press, 2003.

\bibitem{boykov2015volumetric}
Y.~Boykov, H.~Isack, C.~Olsson, and I.~Ben~Ayed, ``Volumetric bias in
  segmentation and reconstruction: Secrets and solutions,'' in
  \emph{International Conference on Computer Vision (ICCV)}, 2015, pp.
  1769--1777.

\bibitem{lange2000optimization}
K.~Lange, D.~R. Hunter, and I.~Yang, ``Optimization transfer using surrogate
  objective functions,'' \emph{Journal of computational and graphical
  statistics}, vol.~9, no.~1, pp. 1--20, 2000.

\bibitem{Zhang2007}
Z.~Zhang, J.~T. Kwok, and D.-Y. Yeung, ``Surrogate maximization/minimization
  algorithms and extensions,'' \emph{Machine Learning}, vol.~69, pp. 1--33,
  2007.

\bibitem{Yuille2001}
A.~L. Yuille and A.~Rangarajan, ``The concave-convex procedure {(CCCP)},'' in
  \emph{Neural Information Processing Systems ({NIPS})}, 2001, pp. 1033--1040.

\bibitem{Narasimhan2005}
\BIBentryALTinterwordspacing
M.~Narasimhan and J.~Bilmes, ``A submodular-supermodular procedure with
  applications to discriminative structure learning,'' in \emph{Conference on
  Uncertainty in Artificial Intelligence (UAI)}, 2005, pp. 404--412. [Online].
  Available: \url{http://dl.acm.org/citation.cfm?id=3020336.3020387}
\BIBentrySTDinterwordspacing

\bibitem{Boyd_EE364b}
\BIBentryALTinterwordspacing
S.~Boyd, ``{EE364b} lecture notes in alternating direction method of
  multipliers,'' 2018. [Online]. Available:
  \url{https://web.stanford.edu/class/ee364b/lectures/admm_slides.pdf}
\BIBentrySTDinterwordspacing

\bibitem{Yuan2017}
J.~Yuan, K.~Yin, Y.~Bai, X.~Feng, and X.~Tai, ``Bregman-proximal augmented
  lagrangian approach to multiphase image segmentation,'' in \emph{Scale Space
  and Variational Methods in Computer Vision (SSVM)}, 2017, pp. 524--534.

\bibitem{csiszar_korner_2011}
I.~Csiszár and J.~Körner, \emph{Information Theory: Coding Theorems for
  Discrete Memoryless Systems}, 2nd~ed.\hskip 1em plus 0.5em minus 0.4em\relax
  Cambridge University Press, 2011.

\bibitem{Srivastava14}
N.~Srivastava, G.~Hinton, A.~Krizhevsky, I.~Sutskever, and R.~Salakhutdinov,
  ``Dropout: A simple way to prevent neural networks from overfitting,''
  \emph{Journal of Machine Learning Research}, vol.~15, no.~1, pp. 1929--1958,
  2014.

\bibitem{Nair10}
V.~Nair and G.~E. Hinton, ``Rectified linear units improve restricted boltzmann
  machines,'' in \emph{International Conference on Machine Learning (ICML)},
  2010, pp. 807--814.

\bibitem{Valpola15}
H.~Valpola, ``From neural {PCA} to deep unsupervised learning,'' \emph{arXiv:
  1411.7783}, 2015.

\bibitem{glorot10a}
X.~Glorot and Y.~Bengio, ``Understanding the difficulty of training deep
  feedforward neural networks,'' in \emph{International Conference on
  Artificial Intelligence and Statistics (AISTATS)}, 2010, pp. 249--256.

\bibitem{kingma2014adam}
D.~P. Kingma and J.~Ba, ``Adam: A method for stochastic optimization,''
  \emph{ICLR}, 2015.

\bibitem{Kuhn55}
H.~W. Kuhn and B.~Yaw, ``The hungarian method for the assignment problem,''
  \emph{Naval Res. Logist. Quart}, pp. 83--97, 1955.

\bibitem{Xu03}
W.~Xu, X.~Liu, and Y.~Gong, ``Document clustering based on non-negative matrix
  factorization,'' in \emph{International ACM SIGIR Conference on Research and
  Development in Information Retrieval}, 2003, pp. 267--273.

\bibitem{Taylor2016}
G.~Taylor, R.~Burmeister, Z.~Xu, B.~Singh, A.~Patel, and T.~Goldstein,
  ``Training neural networks without gradients: A scalable {ADMM} approach,''
  in \emph{International Conference on Machine Learning ({ICML})}, 2016, pp.
  2722--2731.

\bibitem{Wang2019}
J.~Wang, F.~Yu, X.~Chen, and L.~Zhao, ``{ADMM} for efficient deep learning with
  global convergence,'' in \emph{ACM SIGKDD International Conference on
  Knowledge Discovery \& Data Mining ({KDD})}, 2019.

\bibitem{Ziko2018}
I.~M. Ziko, E.~Granger, and I.~{Ben Ayed}, ``Scalable laplacian k-modes,'' in
  \emph{Neural Information Processing Systems (NeurIPS)}, 2018.

\bibitem{Shaham18}
U.~Shaham, K.~Stanton, H.~Li, B.~Nadler, R.~Basri, and Y.~Kluger,
  ``{SpectralNet}: Spectral clustering using deep neural networks,'' in
  \emph{International Conference on Learning Representation (ICLR)}, 2018.

\bibitem{Kailath67}
T.~{Kailath}, ``The divergence and {Bhattacharyya} distance measures in signal
  selection,'' \emph{IEEE Transactions on Communication Technology}, vol.~15,
  no.~1, pp. 52--60, February 1967.

\bibitem{BERTSEKAS76}
D.~P. Bertsekas, ``On penalty and multiplier methods for constrained
  minimization,'' \emph{{SIAM} J. control and optimization}, vol.~14, no.~2,
  1976.

\end{thebibliography}
\bibliographystyle{IEEEtran}

\end{document}